\documentclass[10pt]{article}
\usepackage{amsfonts}
\usepackage{amsmath}
\usepackage{amssymb}
\usepackage{amsthm}
\usepackage{mathrsfs}
\usepackage{graphicx}
\usepackage{tabularx}
\usepackage[includemp,body={398pt,550pt},footskip=30pt,%
            marginparwidth=60pt,marginparsep=10pt]{geometry}

\numberwithin{equation}{section}  
\numberwithin{figure}{section} 
\newtheorem{proposition}{P{\scriptsize ROPOSITION}}[section]
\newtheorem{theorem}[proposition]{T{\scriptsize HEOREM}}
\newtheorem{lemma}[proposition]{L{\scriptsize  EMMA}}

\newtheorem{definition}{D{\scriptsize  EFINITION}}[section]

\newtheorem{example}{E{\scriptsize  xample}}[section]
\newtheorem{model}{M{\scriptsize odel}}

\def\H{\mathbb{H}}

\title {\Large{\bf  
Sample-Relaxed   Two-Dimensional Color Principal Component Analysis for Face Recognition and Image Reconstruction\thanks{This paper was jointly supported by National Natural Science Foundation of China, Grant Nos. 61403155, 61473299, 61375067, 61773384 and 11771188. E-mails: zhaomeixiang2008@126.com (M.-X. Zhao),  zhgjia@jsnu.edu.cn (Z.-G. Jia), dwgong@vip.163.com (D. Gong)}
}}
\author{Mei-Xiang Zhao$^1$,
Zhi-Gang Jia$^{2}$,
Dunwei Gong$^{3}$
 \\
$1,3.$ School of Information and Control Engineering, \\
China University of Mining and Technology,
Xuzhou 221116, China\\
$2.$ School of Mathematics and Statistics,
\\ Jiangsu Normal University,
Xuzhou 221116, China
}

\date{}

\begin{document}
\maketitle

\begin{abstract}
A sample-relaxed  two-dimensional color principal component analysis (SR-2DCPCA) approach  is presented for face recognition and image reconstruction based on quaternion models. A relaxation vector is  automatically generated according to the variances of training color face images with the same label. A sample-relaxed, low-dimensional covariance matrix  is constructed based on all the training samples  relaxed by a relaxation vector, and its eigenvectors corresponding to the $r$ largest eigenvalues are defined as  the optimal projection. The SR-2DCPCA aims to enlarge the global variance rather than  to maximize the variance of the projected training samples. The numerical results based on real face data sets validate that  SR-2DCPCA has a higher recognition rate  than state-of-the-art methods and is efficient in image reconstruction.
\end{abstract}


\section{Introduction}
\noindent
In this paper,  we present a sample-relaxed color two-dimensional principal component analysis (SR-2DCPCA) approach for face recognition and  image reconstruction based on quaternion models.   Different from 2DPCA \cite{yzfy04}  and  2DCPCA  \cite{jlz17}, 
 SR-2DCPCA   utilizes  the variances of training samples with the same label,  aims to maximize the variance of the whole (training and testing) projected  images, and has comparable  feasibility and effectiveness  with state-of-the-art methods.

 Color information is one of the most important characteristics in reflecting structural information of a face image. It can help human being to accurately identify, segment or watermark color images (see \cite{bwzc15}, \cite{cyjz17}-\cite{hoda17}, \cite{shfu07,tjlg15,wlccw15,xlch16,zwsccss16,zzltl15} for example). Various  traditional methods of (grey) face recognition has been improved by fully exploiting color cues. Torres et al. \cite{trl99}  applied the traditional PCA into R, G and B color channels, respectively,  and got a fusion of the  recognition results from  three color channels. 
 Xiang, Yang and Chen \cite{xyc15} proposed a CPCA approach for color face recognition. They utilized a color image matrix-representation model based on the framework of PCA  and applied 2DPCA to compute the optimal projection for feature extraction. 
 Recently,  the quaternion PCA (QPCA) \cite{bisa03},  the two-dimensional QPCA (2DQPCA)  and bidirectional 2DQPCA \cite{scy11},   
 and kernel QPCA (KQPCA) and two-dimensional KQPCA \cite{cyjz17}, 
 have been proposed to generalize the conventional PCA and 2DPCA to color images, using the quaternion representation.  
These approaches have achieved a significant success in promoting the robustness and the ratio of face recognition by utilizing color information.

The PCA-like methods for face recognition are based on linear dimension-reducing projection. The projection directions are exactly orthonormal eigenvectors of  the covariance matrix of training samples, with the aim to maximize the total scatter of the projected training images. Sirovich and Kirby \cite{siki87,kisi90}  applied the PCA approach to the  representation of (grey) face images, and  asserted that  any face image can  be approximately  reconstructed by  a facial basis  and a mean face image.  Based on this assertion,  Turk and Pentland \cite{tupe91} proposed a well-known eigenface method  for  face recognition. Following that,  many properties  of PCA  have been studied and  PCA has gradually become one of the most powerful  approaches  for face recognition  \cite{pent00}.  As a breakthrough development, Yang et al. \cite{yzfy04}  proposed the 2DPCA approach, which constructs the  covariance matrix by directly using 2D face image matrices. Generally,  the  covariance matrix in 2DPCA is of a smaller dimension than that in PCA, which reduces storage and the computational operations. Moreover, the 2DPCA approach is convenient to  apply spacial information of face images, and achieves a higher  face recognition rate than PCA in most cases.

With retaining the advantages of 2DPCA,  2DCPCA recently proposed in \cite{jlz17} is based on quaternion matrices rather than quaternion vectors, and hence can fully utilize color information and the spacial structure of  face images. In this method, the generated covariance matrix is  Hermitian  and low-dimensional. The projection directions are eigenvectors of the covariance matrix corresponding to the $r$ largest eigenvalues.  We find that 2DCPCA ignores label information (if provided) and scatter of training samples with the same label.  If these information is considered when constructing a covariance matrix,  the recognition rate of  2DCPCA  will be improved. As far as our knowledge, there has been no approaches of face recognition  based on the framework of 2DCPCA using these information.

The paper is organized as follows. In section \ref{s:2DCPCA}, we  recall two-dimensional color principle component analysis  approach  proposed in \cite{jlz17}.   In section \ref{s:SR-2DCPCA}, we present a sample-relaxed  two-dimensional color principal component analysis  approach based on quaternion models.  In section \ref{s:varianceupdating}, we provide the theory of the new approach. 
In section \ref{s:experiments}, numerical experiments are conduct by applying the Georgia Tech face  database and the  color FERET face  database. Finally, the conclusion is draw in section \ref{s:conclusion}.

\section{ 2DCPCA}\label{s:2DCPCA}
We firstly recall the two-dimensional color principle component analysis (2DCPCA) from \cite{jlz17}.

 Suppose that an $m\times n$ quaternion matrix is in the form of $Q=Q_0+Q_1i+Q_2j+Q_3k$,   where $i,~j,~k$ are three imaginary values satisfying
   \begin{equation}\label{e1}
   i^2=j^2=k^2=ijk=-1,
\end{equation}   
 and   $Q_s \in \mathbb{R}^{m\times n}$, $s=0,1,2,3$. A  {\it pure quaternion matrix } is a matrix  whose elements are pure quaternions or zero.  In the RGB color space, a  pixel can be represented with a pure quaternion,  $Ri+Gj+Bk$,  where  $R,~G,~B$  stand for the values of Red, Green and Blue components, respectively.  An $m\times n$ color image can be saved as an $m\times n$  pure quaternion matrix, $Q=[q_{ab}]_{m\times n}$,  in which an element,  $q_{ab}=R_{ab}i+G_{ab}j+B_{ab}k$,  denotes one color pixel, and  $R_{ab}$, $G_{ab}$ and $B_{ab}$ are nonnegative integers \cite{pcd03}.  

Suppose that there are $\ell$  training color image samples in total, denoted as $m\times n$ pure quaternion matrices,  $F_1,F_2,...,F_\ell$,  and the mean  image of all the training color images can be defined as follows
 \begin{equation}\label{e:psi}\Psi=\frac{1}{\ell}\sum\limits_{s=1}^{\ell}F_s\in\mathbb{H}^{m\times n}.
 \end{equation} 
 \begin{center}
 \fbox{%
\parbox{0.9\textwidth}{
\begin{model}[\bf 2DCPCA]\label{m:2DCPCA}
\begin{itemize} 
\item[$(1)$] 
Compute the following {\it color image covariance matrix} of training samples 
\begin{equation}\label{e:Gt}
G_t=\frac{1}{\ell}\sum_{s=1}^{\ell}(F_s-\Psi)^*(F_s-\Psi)\in\H^{n\times n}.
\end{equation}
\item[$(2)$] Compute the $r$ $(1\le r\le n)$ largest eigenvalues of $G_t$ and their corresponding  eigenvectors (called eigenfaces), denoted as $(\lambda_1, v_1),$ $\ldots,$ $(\lambda_r, v_r)$.  Let the eigenface subspace
be $V={\rm span}\{v_1,$ $\ldots,$ $v_r\}$.
\item[$(3)$]  Compute the projections of $\ell$ training color face images in the subspace, $V$,
\begin{equation}\label{e:ps4fs}P_s=(F_s-\Psi)V \in\mathbb{H}^{m\times r},\ s=1,\cdots,\ell.\end{equation}
\item[$(4)$] For a given testing sample, $F$,  compute its feature matrix, $P=(F_s-\Psi)V$. 
 Seek the nearest face image, $F_s$ $(1\le s\le \ell)$, whose feature matrix satisfies that
$\|P_s-P\|=\min$.  $F_s$ is output as the person to be recognized.
\end{itemize}
\end{model}
}
}
\end{center}

2DCPCA based on quaternion models can preserve color and spatial information of face images, and its computational complexity of quaternion operations is similar to the computational complexity of real operations  of 2DPCA  (proposed in \cite{yzfy04}).

\section{Sample-relaxed 2DCPCA}\label{s:SR-2DCPCA}
\noindent
In this section,  we present a sample-relaxed 2DCPCA approach based on quaternion models for  face recognition.

Suppose that all the $\ell$ color image samples of the training set can be partitioned into $x$ classes with each containing $\ell_a$ $(a=1,\ldots, x)$ samples:
$$F_1^1,\cdots,F_{\ell_1}^1\ | \ F_1^2,\cdots,F_{\ell_2}^2\ |\  \cdots \ |\  F_1^x,\cdots,F_{\ell_x}^x,$$
 where $F_b^a$ represents the $b$-th sample of the $a$-th class. 
 Now, we define a  relaxation  vector using label  and variance within a class, and  generate a sample-relaxed covariance matrix of training set in the following.

\subsection{The  relaxation vector}\label{ss:relaxvector}
Define the mean image of training samples from the $a$-th class by  
 $$\Psi_a=\frac{1}{\ell_a}\sum\limits_{s=1}^{\ell_a}F_s^a\in\H^{n\times n},$$
 and  
 the $a$-th within-class covariance matrix  by 
 \begin{equation}\label{e:withincovariancematrix}
N_a=\frac{1}{\ell_a}\sum\limits_{s=1}^{\ell_a}(F_s^a-\Psi_a)^*(F_s^a-\Psi_a)\in\H^{n\times n},
\end{equation}
where $a=1,\ldots,x$ and $\sum_{a=1}^{x}\ell_a=\ell$.

The within-class covariance matrix, $N_a$,  is a Hermitian quaternion matrix, and  is semi-definite positive, with its eigenvalues being nonnegative. The maximal eigenvalue of  $N_a$, denoted as $\lambda_{\rm max}(N_a)$,  represents the variance of training samples, 
$F_1^a,\ldots,F_{\ell_a}^a$,  in the principal component. Generally, the larger  $\lambda_{\rm max}(N_a)$ is, the better scattered the training samples of the $a$-th class are. If $\lambda_{\rm max}(N_a)=0$, all the training samples in the $a$-th class are same, and the contribution of the $a$-th class  to the covariance matrix of the training set should be controlled by a small relaxation factor.

Now, we define a  relaxation vector for the training classes.
\begin{definition}\label{d:relvec}  Suppose that the training set has $x$ classes and the covariance matrix of the $a$-th class is $N_a$ $(1\le a\le x)$. Then the relaxation vector can be defined as  
 \begin{equation}\label{e:w}
 W=[w_1,\cdots,w_x]\in\mathbb{R}^n,   
 \end{equation}
 where 
 $$w_a=\frac{e^{\lambda_{\rm max}(N_a)}}{\sum_{b=1}^{x}e^{\lambda_{\rm max}(N_b)}}$$
 is  the relaxation factor of the $a$-th class. 
 \end{definition}
 %
  If the $a$-th class contains   $\ell_a~(\ge 1)$  training samples,  the relaxation factor of each sample   can be calculated as $w_a/\ell_a$. If each training class has only one  sample, i.e.,  $\ell_1=\ldots=\ell_x=1$,   all the within-class covariance matrices will be zero matrices,  and  
$$\lambda_{\rm max}(N_1)=\cdots=\lambda_{\rm max}(N_x)=0.$$  
In this case,  the relaxation factor of  each  class will be $w_a=1/x$,  and so will its unique sample.

\subsection{The relaxed covariance  matrix and eigenface subspace}\label{ss:covariance}
Now,  we define the relaxed covariance matrix of training set. 
\begin{definition}\label{d:Gw}
With  the  relaxation vector,   $W=[w_1,\cdots,w_n]^T\in\mathbb{R}^n$,  defined by  \eqref{e:w}, 
  the  {\it relaxed covariance matrix} of training set can be defined as  follows
\begin{equation}\label{e:Gw}
G_w=\frac{1}{\ell}\sum_{a=1}^x\left(\frac{w_a}{\ell_a}\sum_{s=1}^{\ell_a}(F_s^a-\Psi)^*(F_s^a-\Psi) \right)\in\H^{n\times n},
\end{equation}
where  $\ell_a$ means the number of training samples in the $a$-th class,  $\sum_{a=1}^x\ell_a=\ell$, and $\Psi$ is defined by \eqref{e:psi}.
\end{definition}
If there is no label information, or people are unwilling to use such information,   the training set can be regarded as containing $\ell$ classes,  with each having one sample, i.e.,  $x=\ell$ and $\ell_a=1$.  In this case,  the relaxed covariance matrix,  $G_w$, defined by \eqref{e:Gw} is exactly the covariance matrix, $G_t$, defined by \eqref{e:Gt}.

\begin{definition}\label{d:JV}Suppose that $G_w$ is the relaxed covariance matrix of training set.  The {\it generalized total scatter criterion} can be defined as follows
\begin{equation}\label{e:gtsc}J(V)={\rm trace}(V^*G_wV)=\sum_{s=1}^{r}v_s^*G_wv_s,\end{equation}
where  $V=[v_1,\cdots, v_r]\in \mathbb{H}^{n\times r}$ is  a  quaternion matrix with unitary column  vectors.
\end{definition}
  Note that  $J(V)$ is real and nonnegative since  the quaternion matrix, $G_w$, is Hermitian and positive semi-definite. Our aim is to select the  orthogonal  projection axes, $v_1,\ldots,v_r$,   to maximize the criterion, $J(V)$, i.e.,  
\begin{equation}
\begin{array}{l}
\{v_1^{opt},\ldots,v_r^{opt}\}={\rm arg}\max \sum_{s=1}^{r}v_s^*G_wv_s\\
\ {\rm s.t.}\ \ 
v_s^*v_t=\left\{
\begin{array}{l}  1,~ s=t,\\  0,~s\neq t, \end{array}
~s,t=1,\cdots,r.
\right.                
\end{array} 
\end{equation}
 These optimal projection axes are in fact the orthonormal eigenvectors of $G_w$ corresponding to the  $r$ largest eigenvalues.

Once the relaxed covariance matrix, $G_w$,  is built, we can compute its $r$ $(1\le r\le n)$ largest (positive) eigenvalues  and their corresponding  eigenvectors (called eigenfaces), denoted as $(\lambda_1, v_1^{opt}),\ldots, (\lambda_r, v_r^{opt})$.   
Let the optimal projection be  
\begin{equation}\label{e:V}V=[v_1^{opt},\cdots, v_r^{opt}]\end{equation}
  and the diagonal matrix be
  \begin{equation}\label{e:D}D={\rm diag}(\lambda_1,\ldots, \lambda_r). \end{equation}
Then $V^*V=I_r$ and $D>0$.  Following that, for any quaternion matrix norm $\|\cdot\|$, we define the $D$-norm $\|\cdot\|_D$ as follows
$$\|A\|_D=\|AD\|,~\text{\rm for any quaternion matrix,}~A\in\mathbb{H}^{m\times r}.$$

  \subsection{Color face recognition}\label{ss:eigenfaces}
In this section, we apply the optimal projection, $V$, and the positive definite matrix, $D$,   to color face recognition.  

The projections of $\ell$ training color face images in the subspace, $V$, are
\begin{equation}\label{e:ps4fs}P_s=(F_s-\Psi)V \in\mathbb{H}^{m\times r},\ s=1,\cdots,\ell,
\end{equation}
where $\Psi$ is defined by \eqref{e:psi}. The columns of  $P_s$,  $y_t=(F_s-\Psi)v_t$, $t=1,\ldots,r$,  are called the {\it principal component (vectors)},   and   $P_s$ is called the {\it feature matrix} or {\it feature image} of the sample image, $F_s$.    
Each {\it principal component} of the sample-relaxed 2DCPCA is a quaternion vector.
With the feature matrices in hand, we use a nearest neighbour classifier for color face recognition. 

Now, we present a sample-relaxed   two-dimensional color principal component analysis (SR-2DCPCA) approach.
\begin{center}
   \fbox{%
\parbox{0.9\textwidth}{
\begin{model}[\bf SR-2DCPCA]\label{m:sr2dpca} 

\begin{itemize} 
\item[$(1)$] 
Compute the  relaxed color image covariance matrix, $G_w$, of training samples   by 
\eqref{e:Gw}
\item[$(2)$] Compute the $r$ $(1\le r\le n)$ largest eigenvalues of $G_w$ and their corresponding  eigenvectors (called eigenfaces), denoted as $(\lambda_1, v_1),$ $\ldots,$ $(\lambda_r, v_r)$.  
 Let the eigenface subspace be $V={\rm span}\{v_1,$ $\ldots,$ $v_r\}$ and the weighted matrix be $D={\rm diag}(\lambda_1,\cdots,\lambda_r)$. 
\item[$(3)$]  Compute the projections of $\ell$ training color face images in the subspace, $V$,
\begin{equation*}P_s=(F_s-\Psi)V \in\mathbb{H}^{m\times r},\ s=1,\cdots,\ell.\end{equation*}
\item[$(4)$] For a testing sample, $F$,  compute its feature matrix, $P=(F_s-\Psi)V$. 
 Seek the nearest face image, $F_s$ $(1\le s\le \ell)$ whose feature matrix satisfies that
$\|P_s-P\|_D=\min$.  $F_s$ is output as the person to be recognized.
\end{itemize}
 \end{model}
 }
 }
 \end{center}

SR-2DCPCA can preserve color and spatial information of color face images as 2DCPCA (proposed in \cite{jlz17}).    
Compared to 2DCPCA,   SR-2DCPCA  has an additional  computation  amount on calculating  the relaxation vector, and generally  provides a better discriminant for classification.  The aim of the projection of SR-2DCPCA is to maximize the variance of the whole samples, while that of 2DCPCA is to maximize the variance of training samples.   Note that the eigenvalue problems of Hermitian quaternion matrices in Model \ref{m:sr2dpca} can be solved by the fast structure-preserving algorithm  proposed in  \cite{jwl13}.

Now we provide a toy example.
\begin{example}\label{ex:toy}  As shown in  Figure \ref{f:toyexam}, $400$ randomly generated points  are equally separated into two classes (denoted as $\times$ and $\circ$, respectively).  We choose 100 points from each class as training samples (denoted as magenta $\times$ and $\circ$ ) and the rest as testing samples (denoted as blue $\times$ and $\circ$).  The principle component of $200$ training points is computed by 2DCPCA and SR-2DCPCA.   In three random cases,  the relaxation vectors are  $[0.4771,\ 0.5228]$, $[0.6558,\ 0.3442]$  and $[0.5097,\ 0.4903]$;   the computed principle components are plotted with the blue lines. 
  The variances of the training set and the whole $400$ points, under the projection of 2DCPCA and SR-2DCPCA,  are shown in the following table.
  \begin{table*}[htbp]
  \center
\begin{tabular}{|c|cc|cc|cc|}
  \hline
Case &  \multicolumn{2}{c|}{  Variance of training points} &  \multicolumn{2}{c|}{  Variance of the whole points}       \\ 
           &  {  2DCPCA}     &  { SR-2DCPCA} &  {  2DCPCA}     &  { SR-2DCPCA}       \\ \hline
$1$  &      ${\bf 4.5638} $& $ 4.5562$         &      $4.2452$& ${\bf 4.2617}$            \\   
$2$  &      ${\bf 3.6007}$&$3.5803$     &      $4.1267$&${\bf 4.1459}$                    \\ 
$3$  &      ${\bf 3.6381}$&$3.5582$   &      $3.7255$&${\bf 3.7452}$                      \\ \hline 
\end{tabular}
\end{table*}
 \begin{figure}[!t]
  \centering
      \includegraphics[height=0.320\textwidth,width=0.45\textwidth]{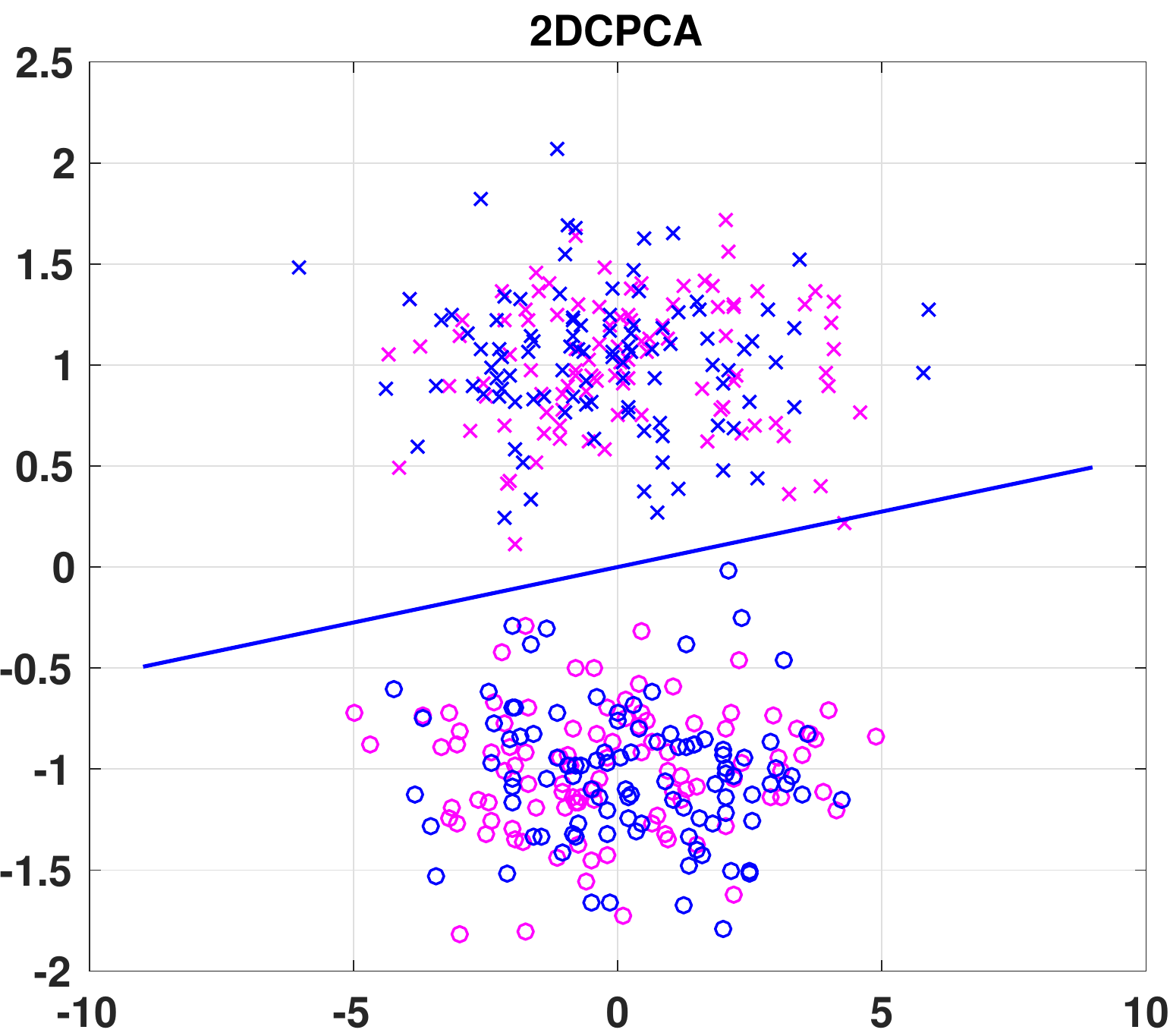}
    \includegraphics[height=0.320\textwidth,width=0.45\textwidth]{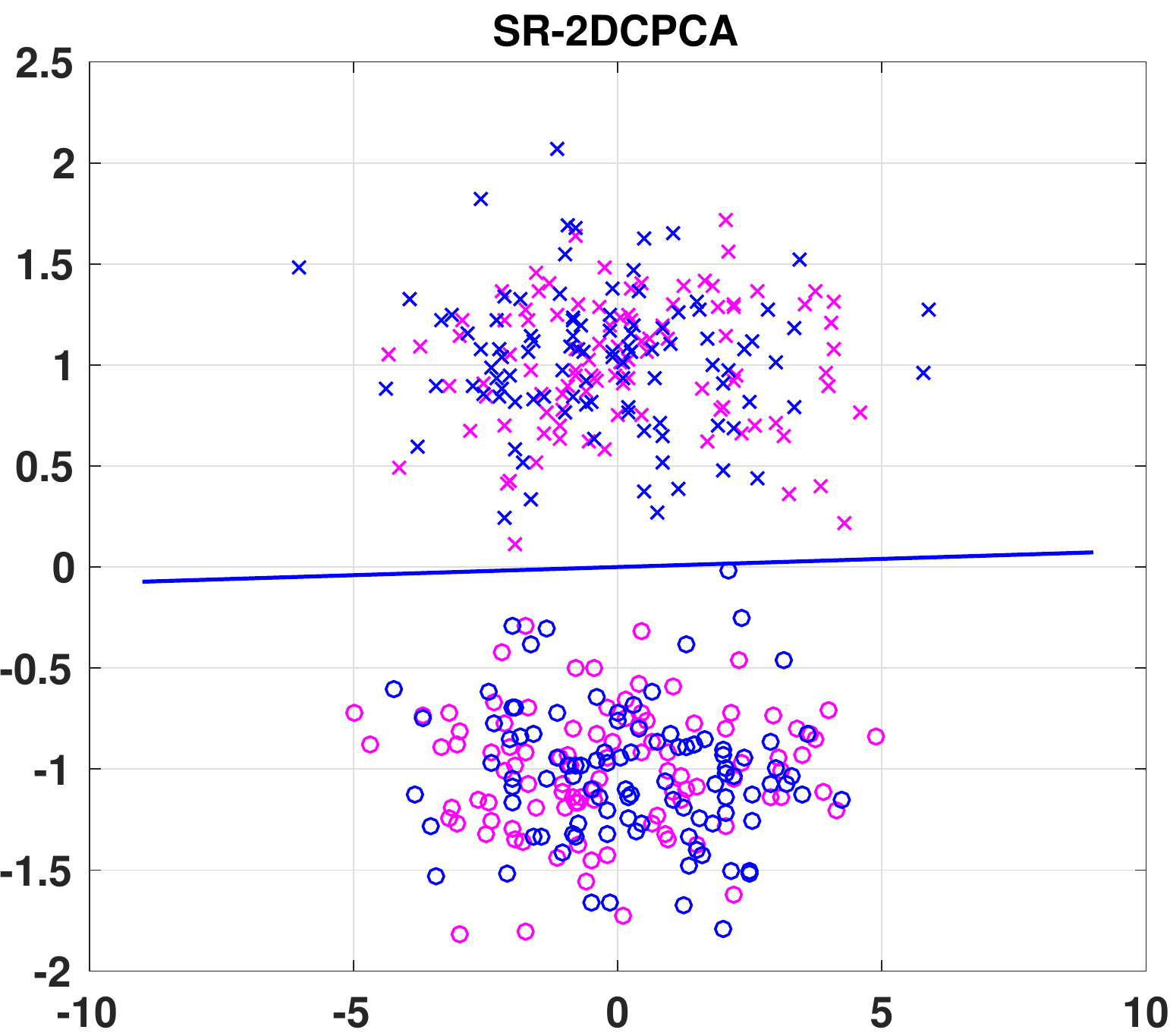}
    \includegraphics[height=0.320\textwidth,width=0.45\textwidth]{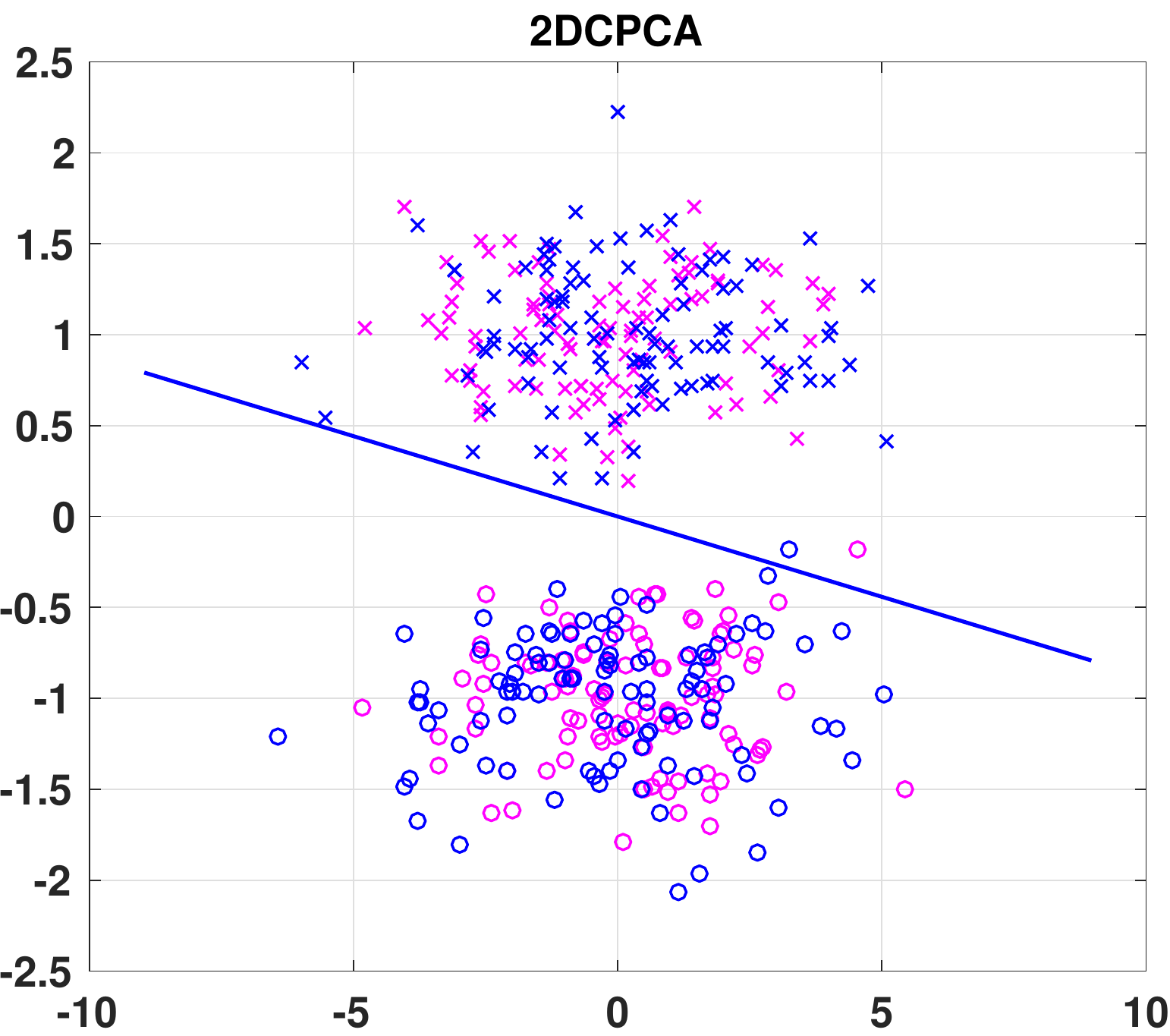}
    \includegraphics[height=0.320\textwidth,width=0.45\textwidth]{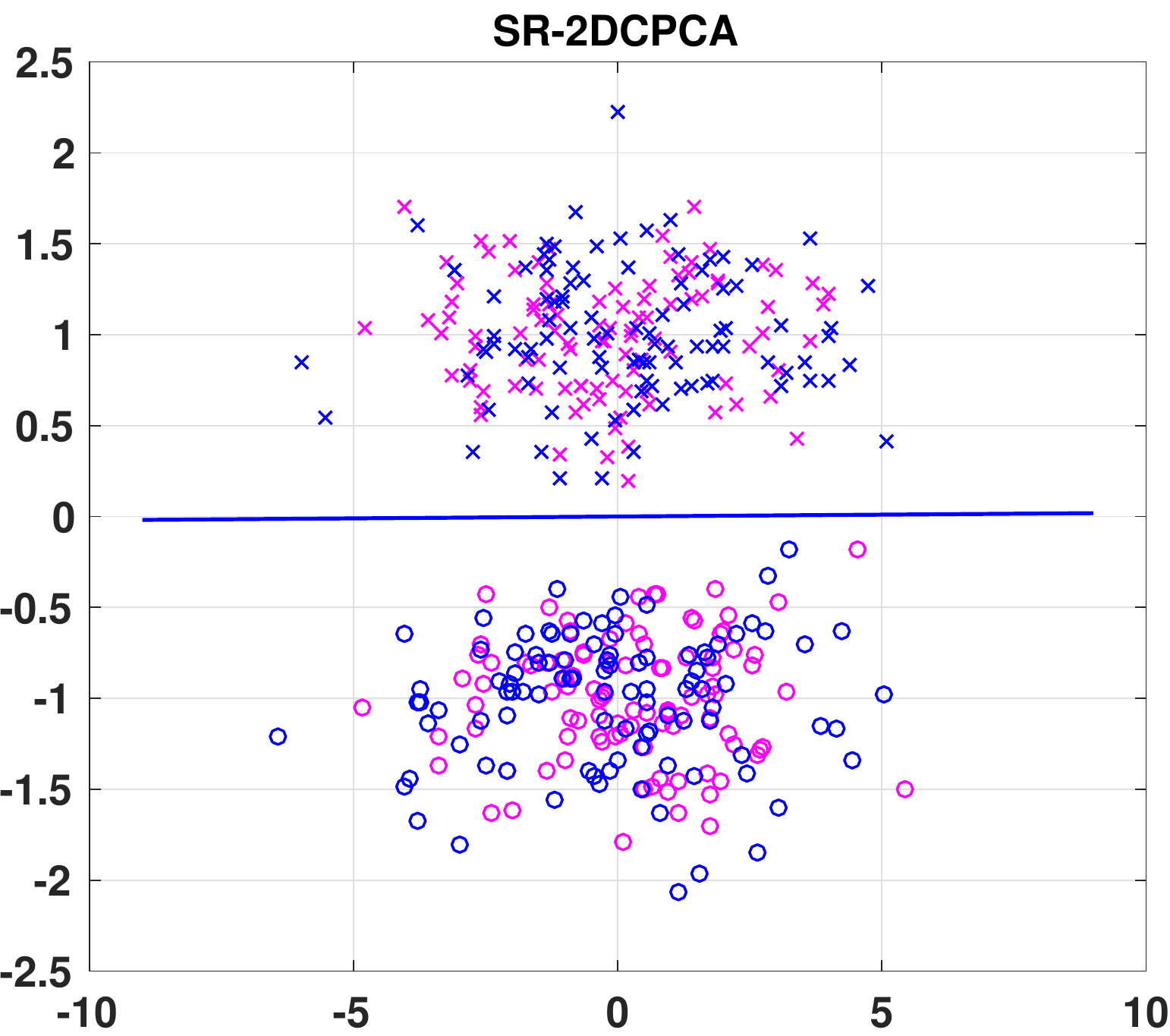}
    \includegraphics[height=0.320\textwidth,width=0.45\textwidth]{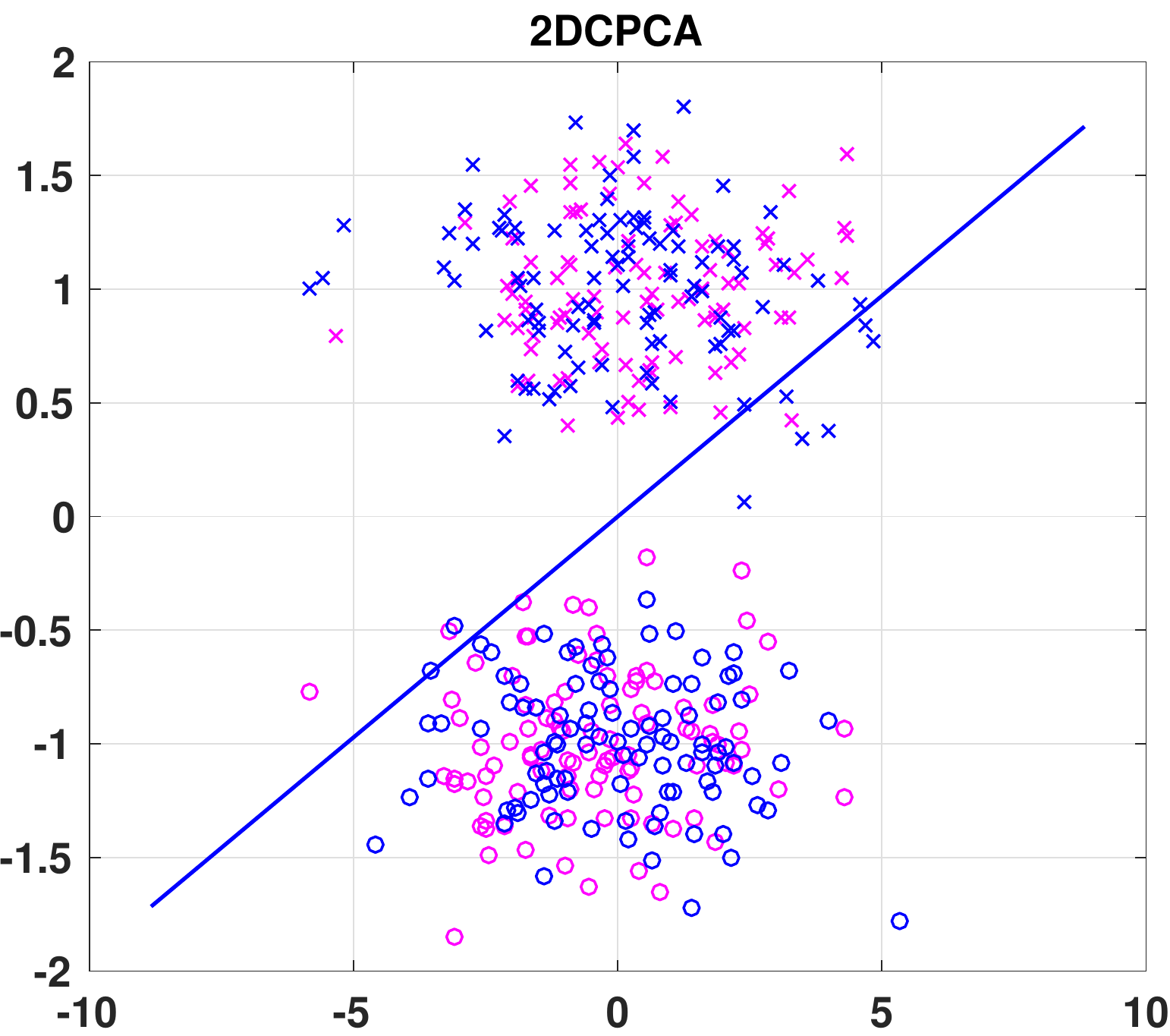}
    \includegraphics[height=0.320\textwidth,width=0.45\textwidth]{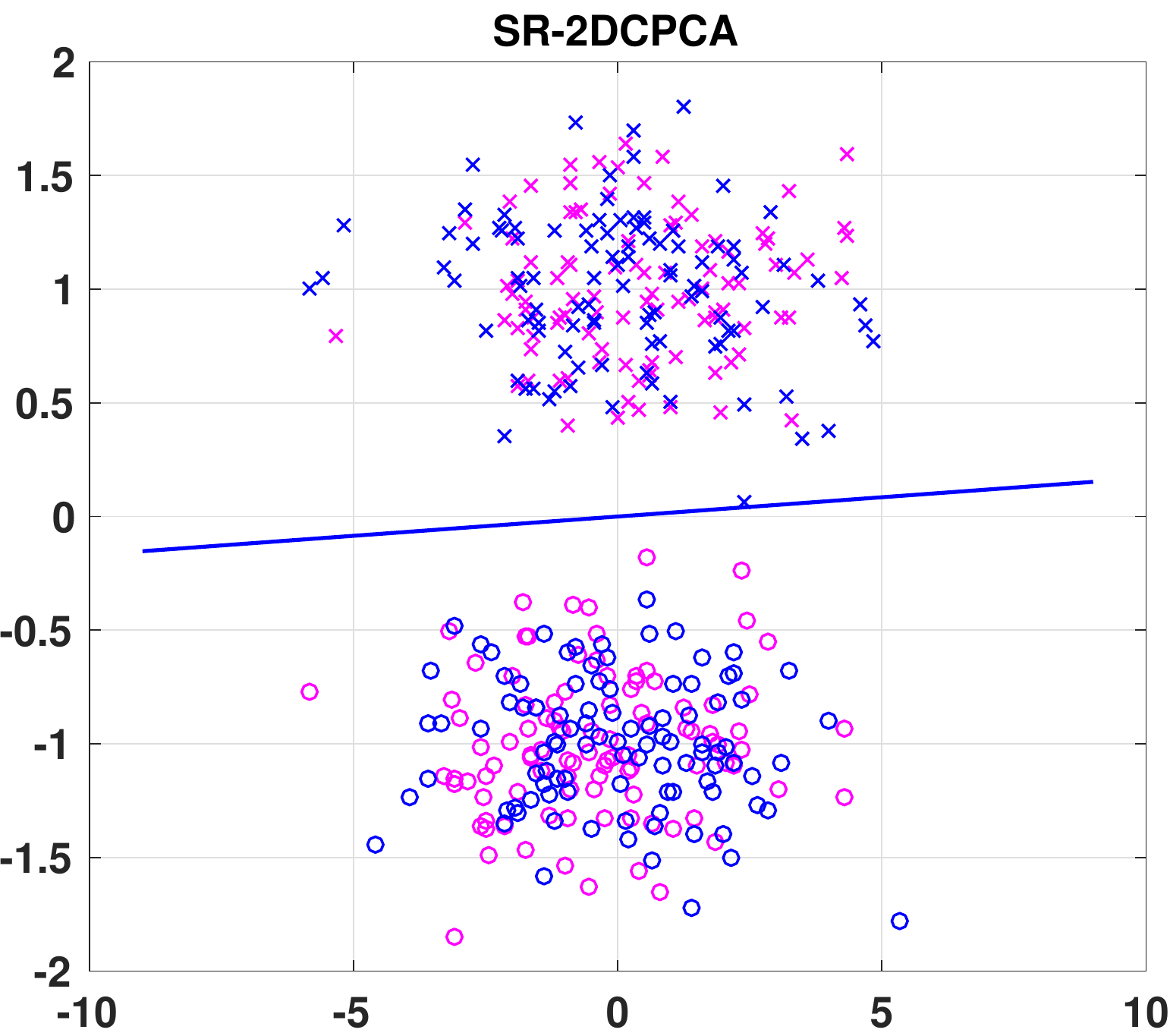}
\caption{Points in three random cases and the principle components (blue lines).}
\label{f:toyexam}
\end{figure}
\end{example}

\subsection{Color image  reconstruction}\label{ss:imreconstruction}

Now, we consider color image reconstruction  based on  SR-2DCPCA.
  Without loss of generality, we suppose that $\Psi=0$. In this case,  the projected color image is  $P_s=F_sV$. 
  \begin{theorem}\label{t:reconstruction}
 Suppose that $V^\bot \in\mathbb{H}^{n\times (n-r)}$ is the unitary complement of~$V$ such that $[V,V^\bot]$ is a unitary quaternion matrix. Then the reconstruction of a training sample, $F_s$, is $P_sV^*$, and  
   \begin{equation}\label{e:ps_ps2}\| P_sV^*- F_s\|=\|F_sV^\bot\|,\ s=1,\cdots,\ell.\end{equation}
 \end{theorem}   
\begin{proof}We only need to prove \eqref{e:ps_ps2}. Since quaternion matrix $[V,V^\bot]$ is unitary,  
$$ VV^*+V^\bot(V^\bot)^*=I_n,
\ V^*V=I_r,
\ (V^\bot)^*V^\bot=I_{n-r},
\ V^*V^\bot=0,$$
where $I_m$ is an $m\times m$ identity matrix.
The projected image, $P_s$,  of the sample image, $F_s$, satisfies that
$$P_sV^*=F_sVV^*=F_s(I_n-V^\bot (V^\bot)^*).$$
Therefore, we have 
$$\| P_sV^*- F_s\|=\|F_sV^\bot(V^\bot)^*\|=\|F_sV^\bot\|.$$
\end{proof}
 \noindent
According to \cite{jlz17}, the image reconstruction rate of $P_sV^*$  can be defined as follows
\begin{equation}\label{e:ratio}Ratio_s=1-\frac{\|P_sV^*-F_s\|_2}{\|F\|_2}.\end{equation}
Since $V$ is generated based on eigenvectors corresponding to the $r$ largest eigenvalues of $G_w$, 
$P_sV^*$ is always a good approximation of the color image, $F_s$.
If the number of chosen principal components $r=n$,  $V$ is a unitary matrix and $Ratio_s=1$, which means $F_s=P_sV^*$.

From the above analysis, SR-2DCPCA is convenient to reconstruct color face images from the projections, as 2DCPCA proposed in \cite{jlz17}.

\section{Variance updating}\label{s:varianceupdating}
In this section we analysis the contribution of training samples from each class to the variance of the projected samples.

For the $b$-th class with $\ell_b$ training samples, let
\begin{eqnarray*}
X&=&\left[(F_1^b-\Psi)^*\ |\ (F_2^b-\Psi)^*\ |\ \cdots\ |\ (F_{\ell_b}^b-\Psi)^*\right],\\
A&=&\frac{1}{\ell}\sum_{a\neq b}\left(\frac{w_a}{\ell_a}\sum_{s=1}^{\ell_a}(F_s^a-\Psi)^*(F_s^a-\Psi) \right).
\end{eqnarray*}
According to the definition of  the relaxed covariance matrix  \eqref{e:Gw},  we can rewrite $G_w$ as
\begin{equation}\label{e:GtD}G_w=A+\frac{w_b}{\ell}(\frac{1}{\ell_b}XX^*).\end{equation}

 Denote all the eigenvalues of an  $n$-by-$n$ Hermitian quaternion matrix, $A$, as $\lambda_1(A), \lambda_2(A),\ldots, \lambda_n(A)$ in descending order. 
\begin{lemma}\label{c:BAX}
  Suppose that $B=A+\rho XX^*$, where $A\in\mathbb{H}^{n\times n}$ is Hermitian, $X\in\mathbb{H}^{n\times m}$  and $\rho\ge 0$.
Then 
$$|\lambda_a(B)-\lambda_a(A)|\le \rho\|X\|_2^2,\ a=1,\cdots,n.$$
 \end{lemma}
\begin{proof} 
The  properties for the eigenvalues of Hermitian (complex) matrices  \cite[Theorem 3.8, page 42]{stew01}  are still right for  Hermitian quaternion matrices. Suppose that $A, \Delta A\in\mathbb{H}^{n\times n}$ are Hermitian matrices, then   
  \begin{equation}\label{e:maxmin}\lambda_a(A)+\lambda_n(\Delta A)\le \lambda_a(A+\Delta A)\le\lambda_a(A)+\lambda_1(\Delta A)
  \end{equation}
for $a=1,\ldots,n$. 
Consequently,
  \begin{equation}\label{e:maxminbound}|\lambda_a(A+\Delta A)-\lambda_a(A)|\le \|\Delta A\|_2,\ a=1,\cdots,n.\end{equation}
Since $XX^*$ is Hermitian and semi-definite positive, its maximal singular value, $\|XX^*\|_2$,  is exactly its maximal eigenvalue,  which is the multiplication of the maximal singular values of $X$ and $X^*$. That is $\|XX^*\|_2=\|X\|_2^2$.
From equation \eqref{e:maxminbound}, we can obtain that 
 $$|\lambda_a(B)-\lambda_a(A)|\le\|\rho XX^*\|_2=\rho\|X\|_2^2,\ a=1,\cdots,n.$$
\end{proof}
 
Since  Hermitian quaternion matrices, $G_w$, $A$ and $XX^*$ are semi-definite positive,  
we obtain that 
$$\lambda_a(A)\le \lambda_a(G_w)\le \lambda_a(A)+ \lambda_1(\frac{w_b}{\ell\ell_b}XX^*).$$
Let the variance of projections of all the training samples on $V$ be 
$$\varepsilon=\sum_{s=1}^r\lambda_s(G_w).$$
The following theorem characterizes  the contribution of training samples from one class  to the variance of all the projected training samples. 
\begin{theorem} With the above notations and a fixed  relaxation vector $W$, the variance of projections of all the training samples  $\varepsilon$ satisfies that
$$\sum_{s=1}^r\lambda_s(A)\le \varepsilon \le \sum_{s=1}^r\lambda_s(A)+\frac{rw_b}{\ell}\|X\|_2^2.$$
\end{theorem}
\begin{proof} From equation \eqref{e:GtD}, we obtain that
$$G_w=A+\rho XX^*,~\rho=\frac{w_b}{\ell\ell_b}>0.$$
Since Hermitian quaternion matrices, $G_w$ and $A$, are semi-definite positive,  we can see that
$$\lambda_s(G_w)\ge \lambda_s(A),$$
and Lemma \ref{c:BAX} implies that
$$\lambda_s(G_w)\le \lambda_s(A)+\frac{w_b}{\ell}\|X\|_2^2.$$
Thus $\sum_{s=1}^r\lambda_s(G_w) \le \sum_{s=1}^r\lambda_s(A)+\frac{rw_b}{\ell}\|X\|_2^2.$
\end{proof}

\section{Experiments}\label{s:experiments}
\noindent
In this section we compare the efficiencies of five  approaches on color face recognition:   
\begin{itemize}
\item  PCA:  the Principle Component Analysis (with converting Color-image into Grayscale),
\item  CPCA:  the Color Principle Component Analysis proposed in {\rm \cite{xyc15}},
\item  2DPCA:   the Two-Dimensional Principle Component Analysis  proposed in  {\rm \cite{yzfy04}} (with converting Color-image into Grayscale), 
\item  2DCPCA: the  Two-Dimensional Color Principle Component Analysis proposed in {\rm \cite{jlz17}},
\item  SR-2DCPCA: the Sample-Relaxed Two-Dimensional Color Principle Component Analysis proposed in Section $\ref{s:SR-2DCPCA}$.
\end{itemize} 
All the numerical experiments are performed with MATLAB-R2016 on a personal computer with   
  Intel(R) Xeon(R) CPU E5-2630 v3 @  2.4GHz (dual processor) and RAM 32GB.

\begin{example}\label{ex:GTFD} In this experiment, we compare SR-2DCPCA  with PCA, CPCA,  2DPCA and 2DCPCA  using  the famous Georgia Tech face database \cite{GTFD}. The database contains various pose faces with various expressions on cluttered backgrounds.  All the images are manually  cropped, and then resized to $33 \times 44$ pixels.  Some  cropped images are shown in Fig. \ref{f:GTFD}.
 \begin{figure}[!t]
           \centering
           \includegraphics[width=0.6\textwidth,height=0.6\textwidth]{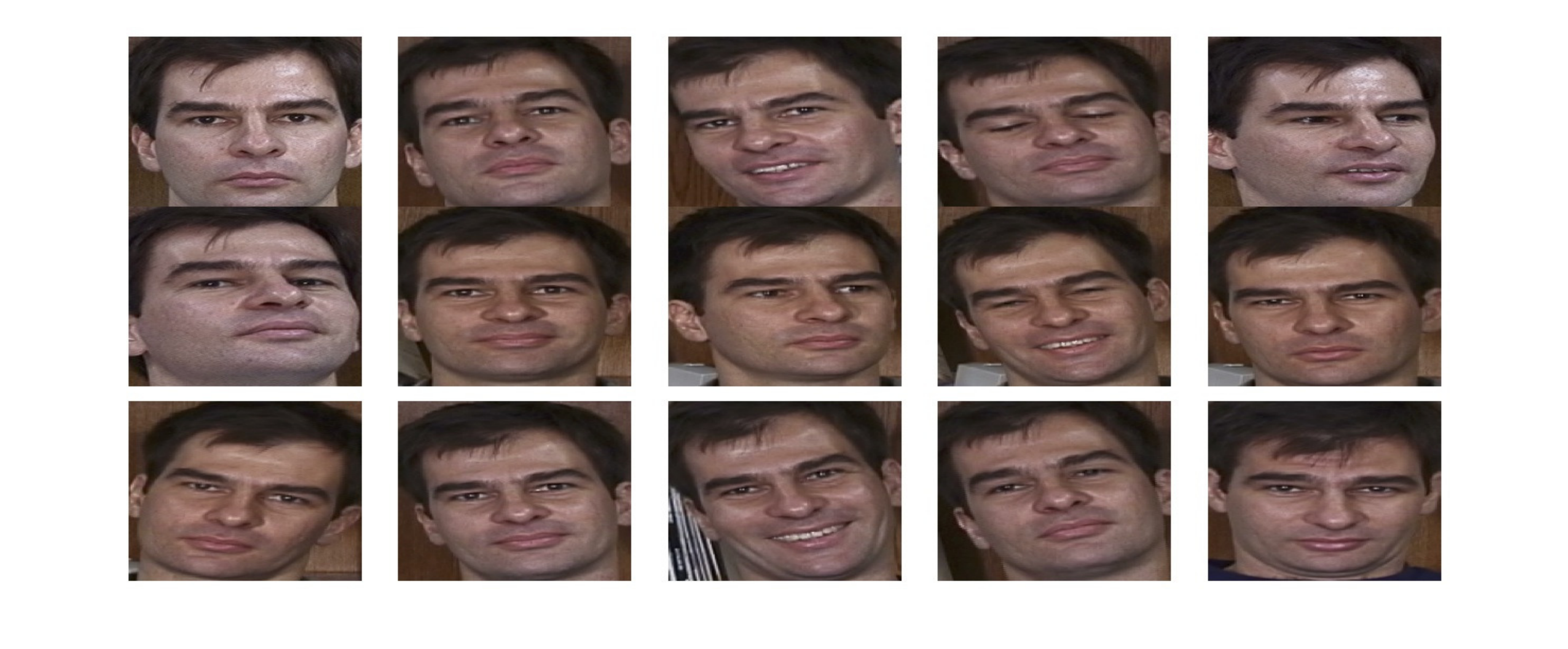} 
           \caption{Sample images for one individual of the Georgia Tech face database. }
           \label{f:GTFD}
\end{figure}
The first $x$ ($=10$ or $13$) images of each individual  are chosen as the training set and the remaining as the testing set. The numbers of the chosen eigenfaces are from $1$ to $20$.  The face recognition rates  of five PCA-like methods are shown in Fig. $\ref{f:RegRate_GTFD_SR-2DCPCA}$.
The top and the below figures show the results for cases that the first $x=10$ and $13$ face images per individual are chosen for training, respectively.

The maximal recognition rate (MR)  and  the CPU time for recognizing one face image in average  are listed in Table $\ref{t:maxratCPUtim}$.  We can see that SR-2DCPCA reaches the highest face recognition rate,  and costs less CPU time than CPCA.

 \begin{figure}[!t]
  \centering
   \includegraphics[height=0.50\textwidth,width=0.8\textwidth]{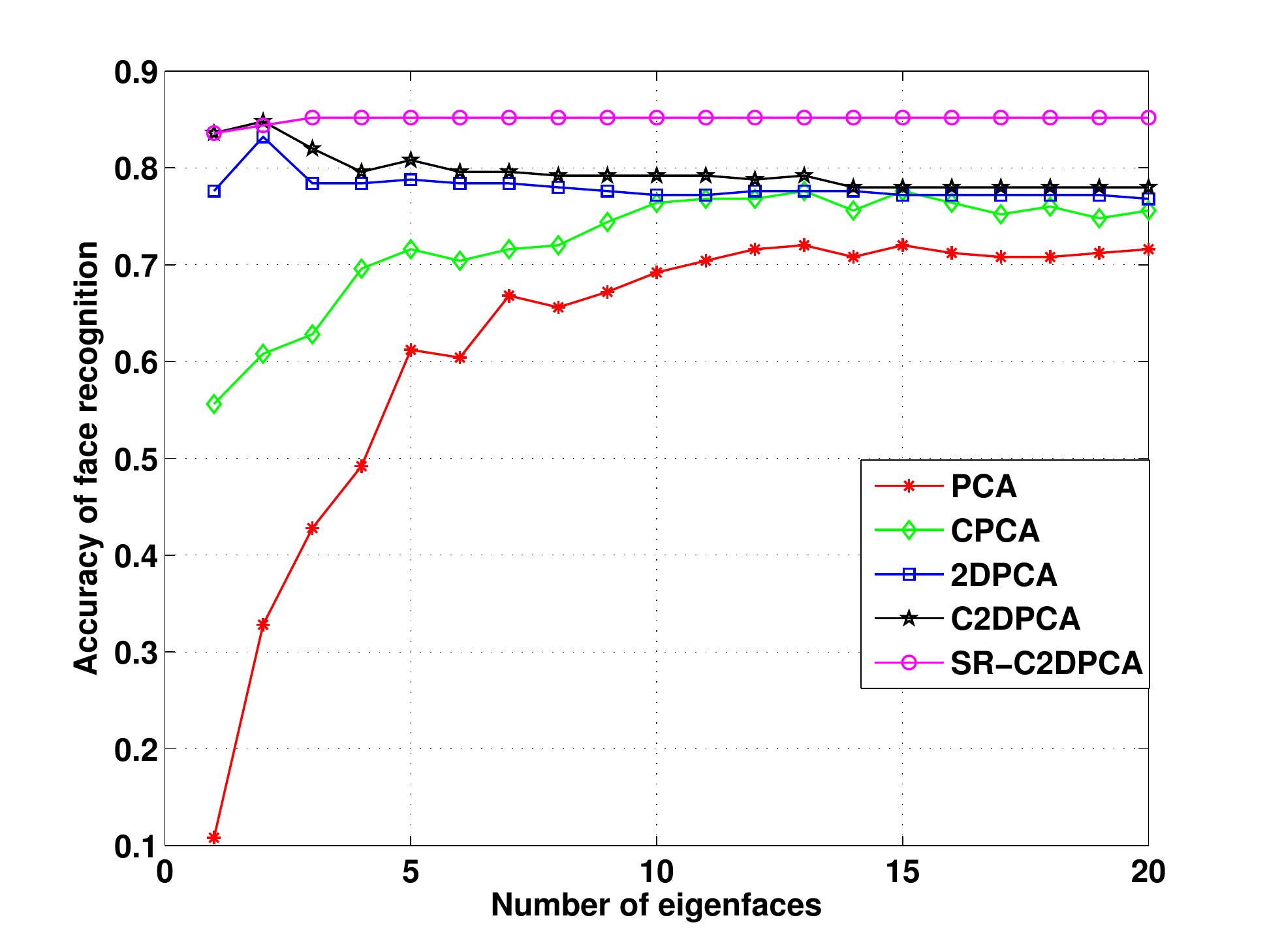}
  \includegraphics[height=0.50\textwidth,width=0.8\textwidth]{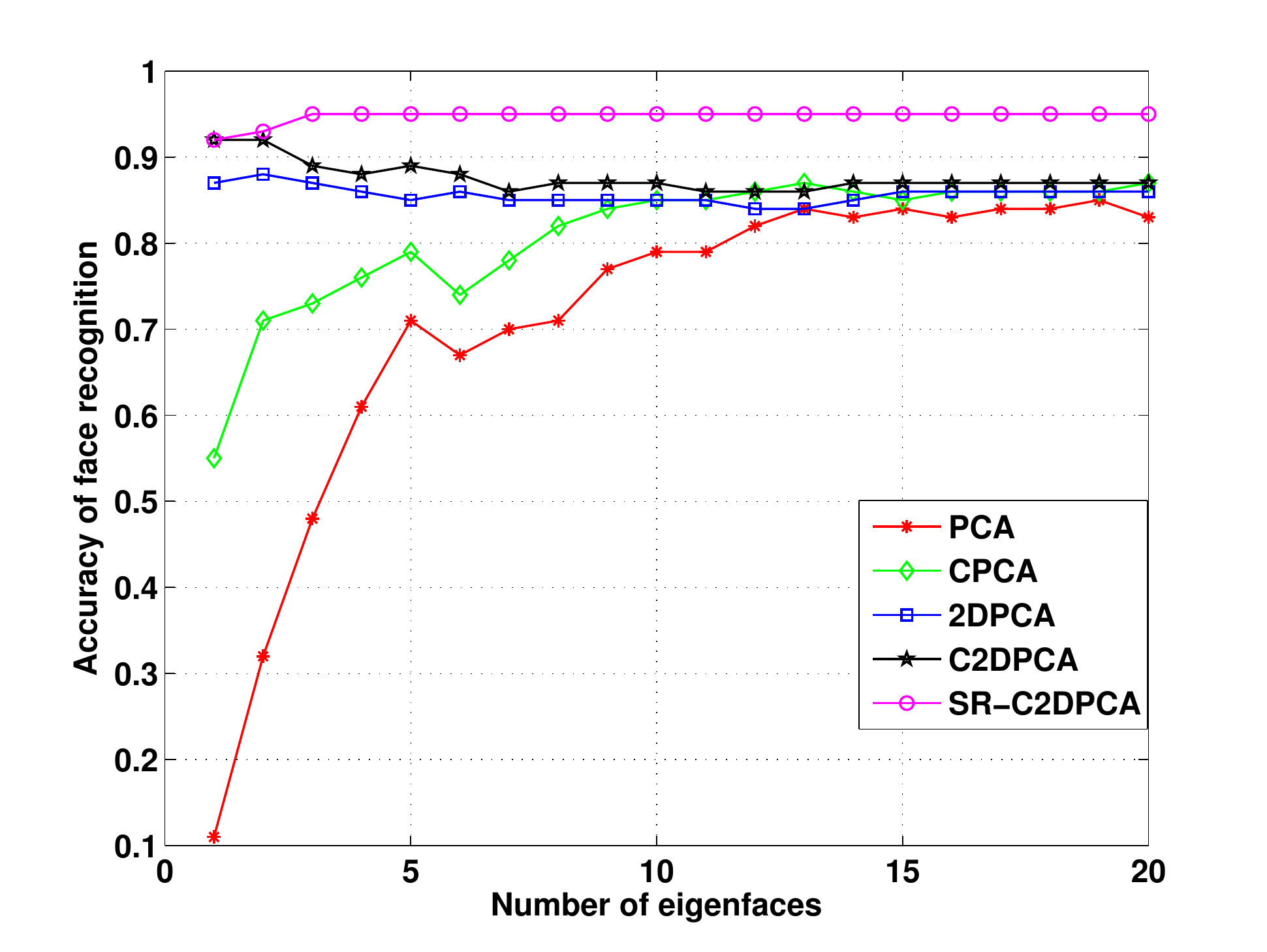}
\caption{Face recognition  rate of PCA-like methods for the Georgia Tech face database.}
\label{f:RegRate_GTFD_SR-2DCPCA}
\end{figure}
\begin{table*}[htbp]
\caption{Maximal face recognition rate  and average CPU time (milliseconds)}\label{t:maxratCPUtim}
\centering
\begin{tabular}{|c|cc|cc|cc|cc|}
  \hline
$x$ &   \multicolumn{2}{c|}{  CPCA}     &  \multicolumn{2}{c|}{  2DPCA}     &   \multicolumn{2}{c|}{ 2DCPCA}  &       \multicolumn{2}{c|}{SR-2DCPCA }    \\ 
&      MR&CPU              &MR&CPU              &MR&CPU                &MR&CPU\\  \hline
$10$  &      $77.6\%$&7.58           &$83.2\%$&0.02&        $84.8\%$&0.09    &      $85.2\%$&0.85\\ 
$13$  &      $87.0\%$&23.80             &$88.0\%$&0.06&       $92.0\%$&0.30    &       $95.0\%$&2.16\\
 \hline
\end{tabular}
\end{table*}
\end{example}

\begin{example}\label{ex:colorferet} In this experiment, we compare three two-dimensional approaches:
  2DPCA, 2DCPCA  and SR-2DCPCA on face recognition and image reconstruction,   by using  the color Face Recognition Technology (FERET)  database \cite{FERET}. 
The database  (version 2, DVD2, thumbnails) contains 269 persons, 3528 color face images,  and each person has various numbers of face images with various backgrounds.  The minimal number of face images for one person is 6, and the maximal one is 44. The size of each cropped color face image  is $192\times 128$ pixels, and
 some samples  are shown
in Fig. \ref{f:colorferet}.
 \begin{figure}[!t]
  \centering
  \includegraphics[width=1.0\textwidth,height=0.8\textwidth]{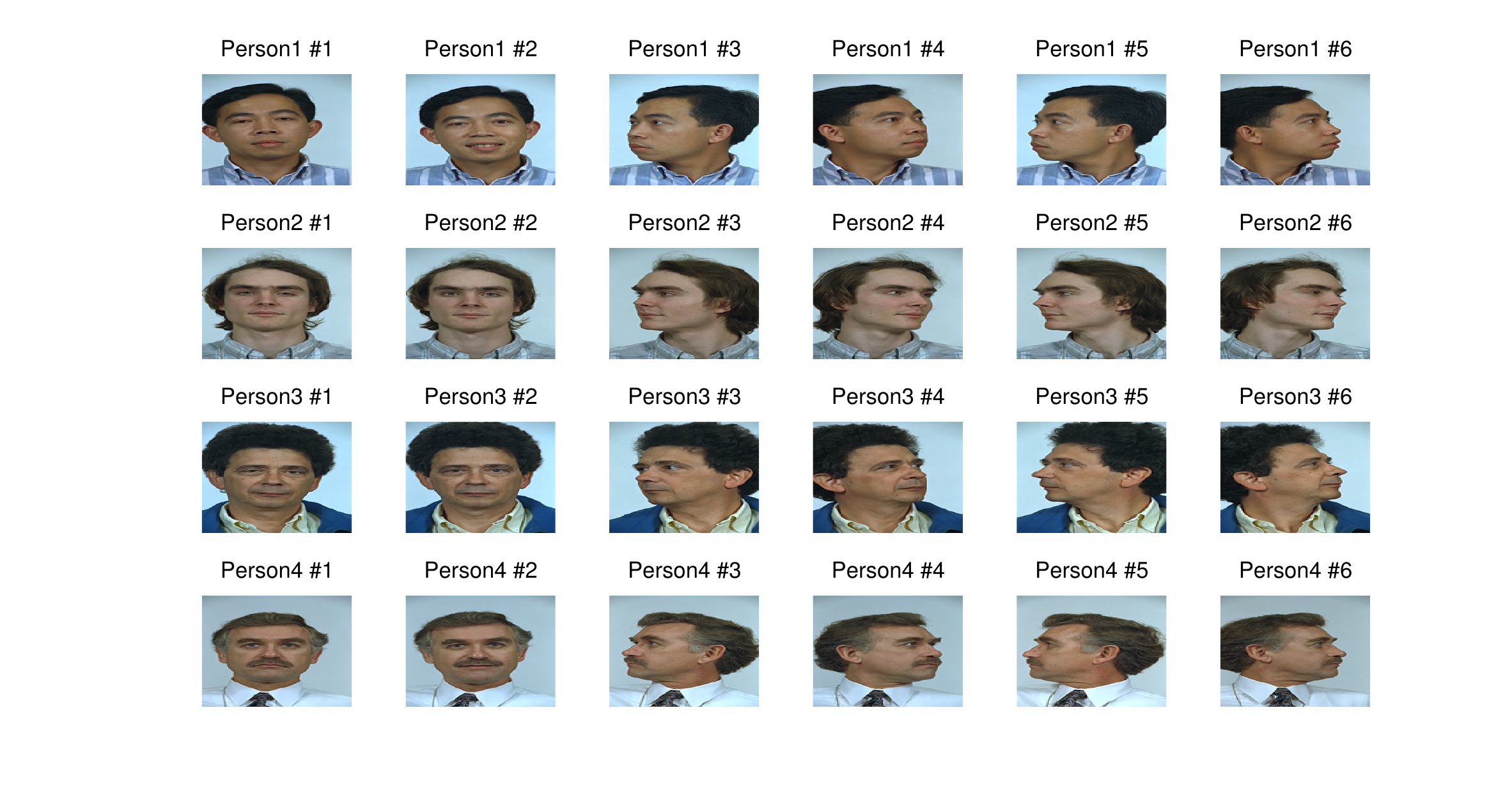} 
\caption{Sample images  of the color FERET face database. }
\label{f:colorferet}
\end{figure}
Taking the first 6 images of each individual as an example,  we randomly choose  $x (=1:5)$  image/images as the training set and the remaining  as the testing set.   This process is repeated 5 times, and the average value is output. For a fixed $x$, we consider ten cases that the number of the chosen eigenfaces increases from $1$ to $10$. The average face recognition rate of these ten cases is shown in the top of Fig. $\ref{f:averageRegRate_FERET_SR-2DCPCA}$. When $x=5$, the face recognition rate with the number of eigenfaces increasing from 1 to 10 is shown in the below of Fig. $\ref{f:averageRegRate_FERET_SR-2DCPCA}$.
 \begin{figure}[!t]
  \centering
    \includegraphics[height=0.5\textwidth,width=0.8\textwidth]{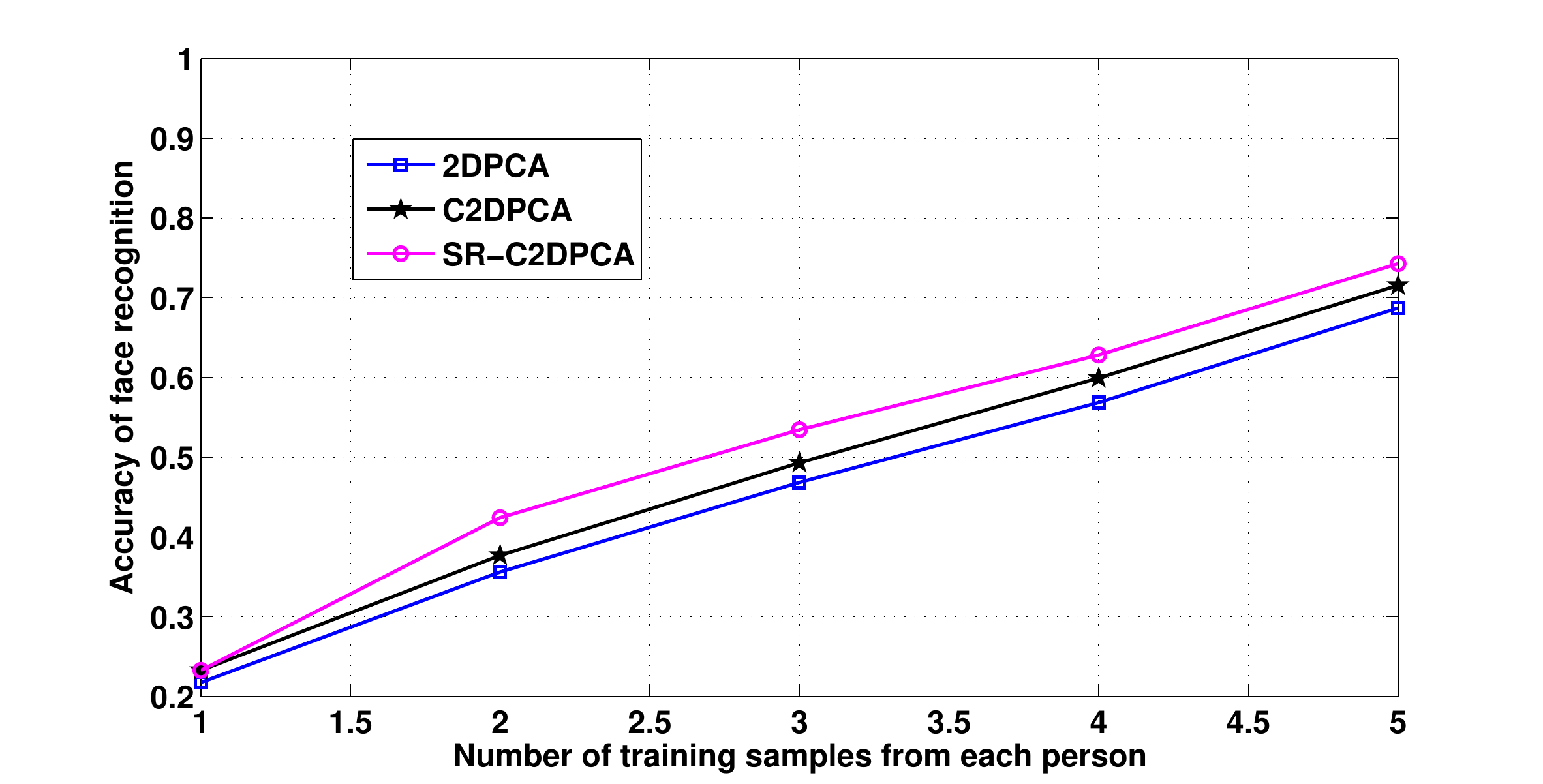}
  \includegraphics[height=0.5\textwidth,width=0.8\textwidth]{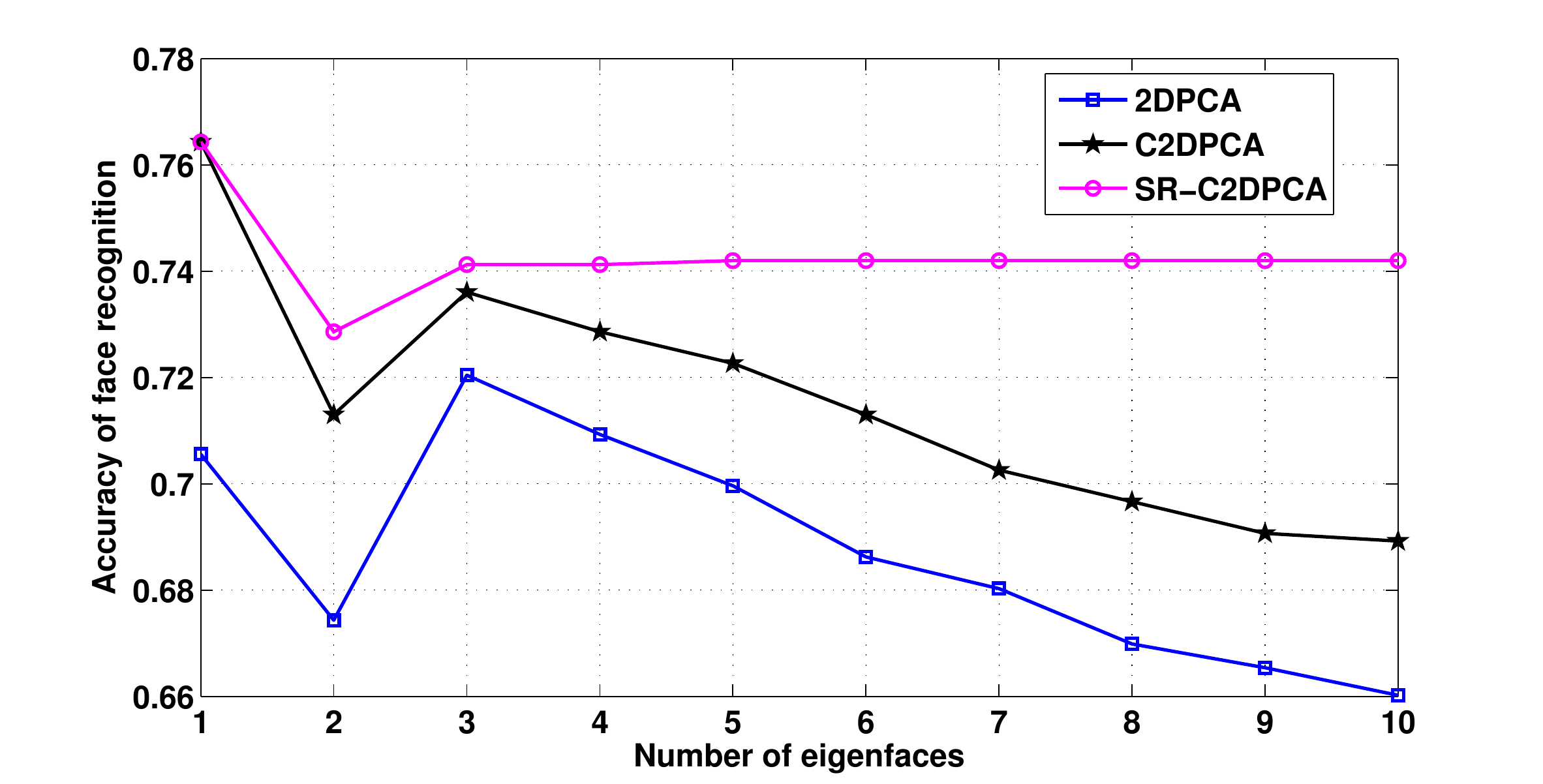}
\caption{Face recognition rate of PCA-like methods  for the color FERET face database.}
\label{f:averageRegRate_FERET_SR-2DCPCA}
\end{figure}

With $x=5$,  we also employ 2DCPCA and SR-2DCPCA to reconstruct the color face images ($F_s$) in training set from  their projections ($P_s$).    
In Fig. $\ref{f:Recon_SR2DCPCA}$,  the reconstruction of the first four persons with 2, 10, 20, 38 and 128 eigenfaces:   the top and the below pictures  are reconstructed  by   2DCPCA and SR-2DCPCA, respectively. 
In Fig. $\ref{f:Recon_ratios}$, we plot the   ratios of image reconstruction by   2DCPCA (top) and SR-2DCPCA (below)  with the number of eigenfaces changing from 1 to 128. The difference in the ratio between the two method is shown in Fig. $\ref{f:Recon_subratios}$ (the ratio of SR-2DCPCA subtracts that of 2DCPCA). These results   indicate that  both 2DCPCA and SR-2DCPCA are convenient to reconstruct color face images from projections, and can reconstruct original color face images with choosing all the eigenvectors to span the eigenface subspace.

\begin{figure*}[!t]
  \centering
    \includegraphics[height=0.5\textwidth,width=0.8\textwidth]{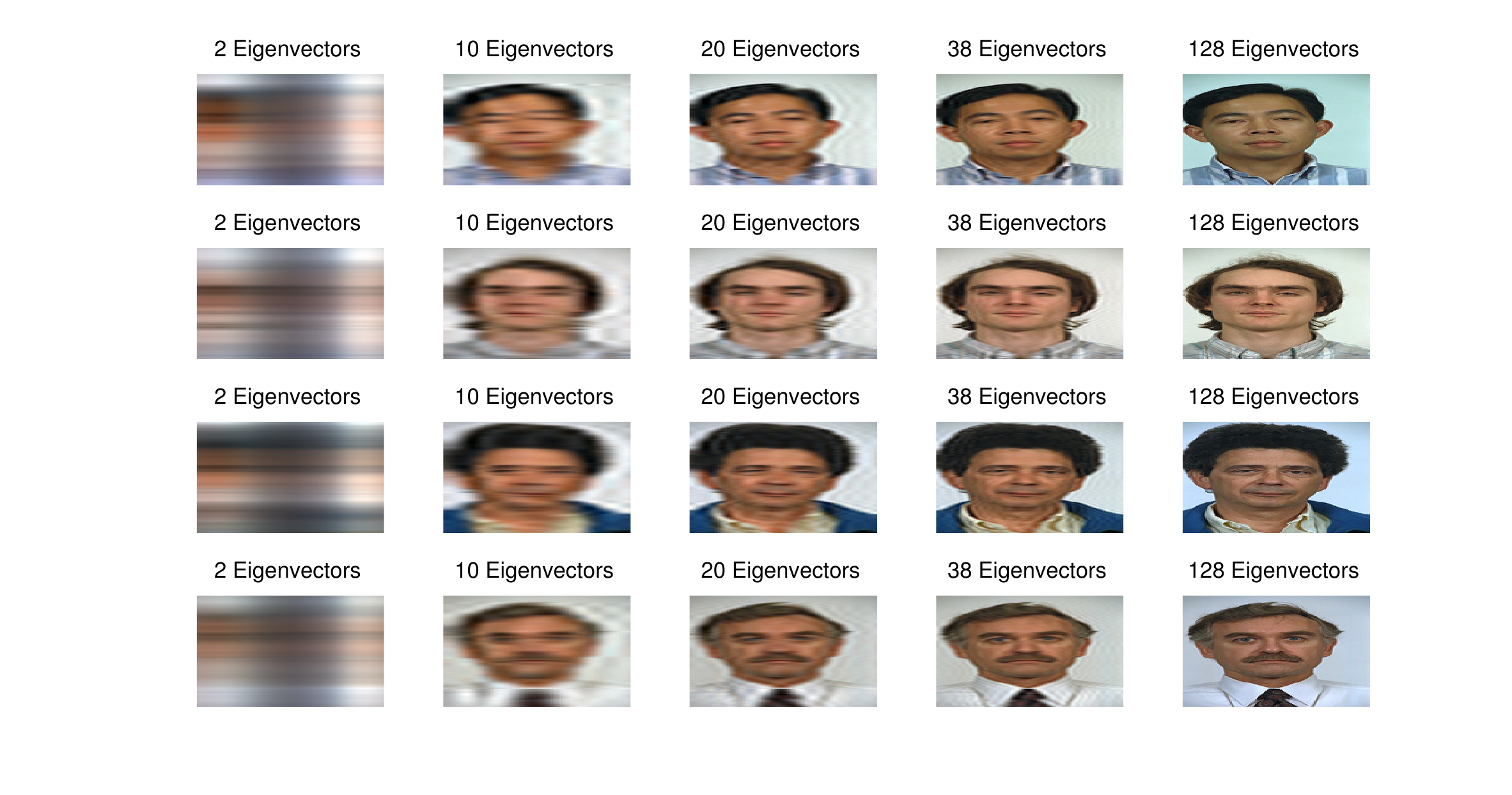}
      \includegraphics[height=0.5\textwidth,width=0.8\textwidth]{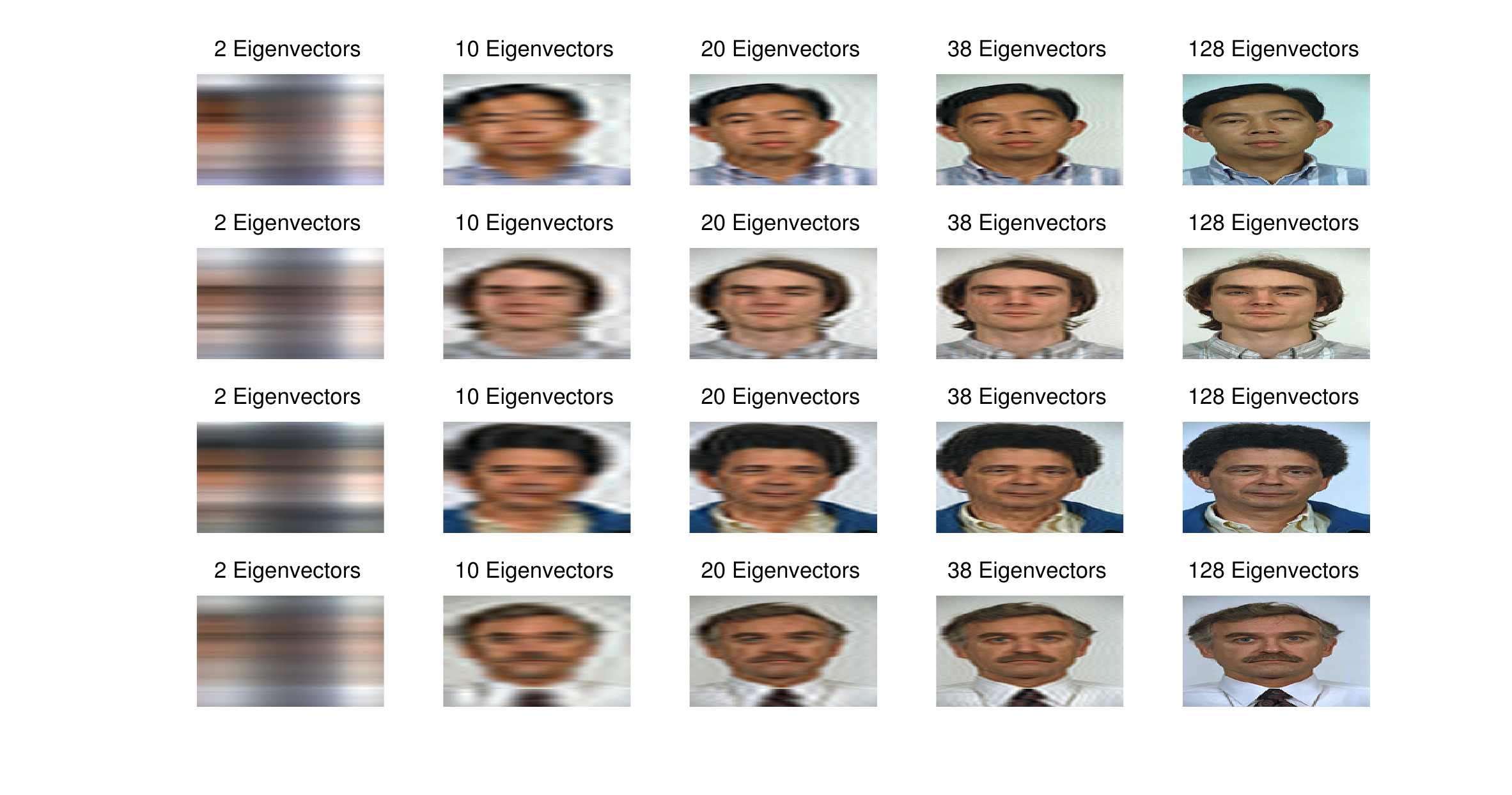}
\caption{Reconstructed  images by 2DCPCA (top) and SR-2DCPCA (below).}
\label{f:Recon_SR2DCPCA}

\end{figure*}

 \begin{figure*}[!t]
  \centering
    \includegraphics[height=0.5\textwidth,width=0.8\textwidth] {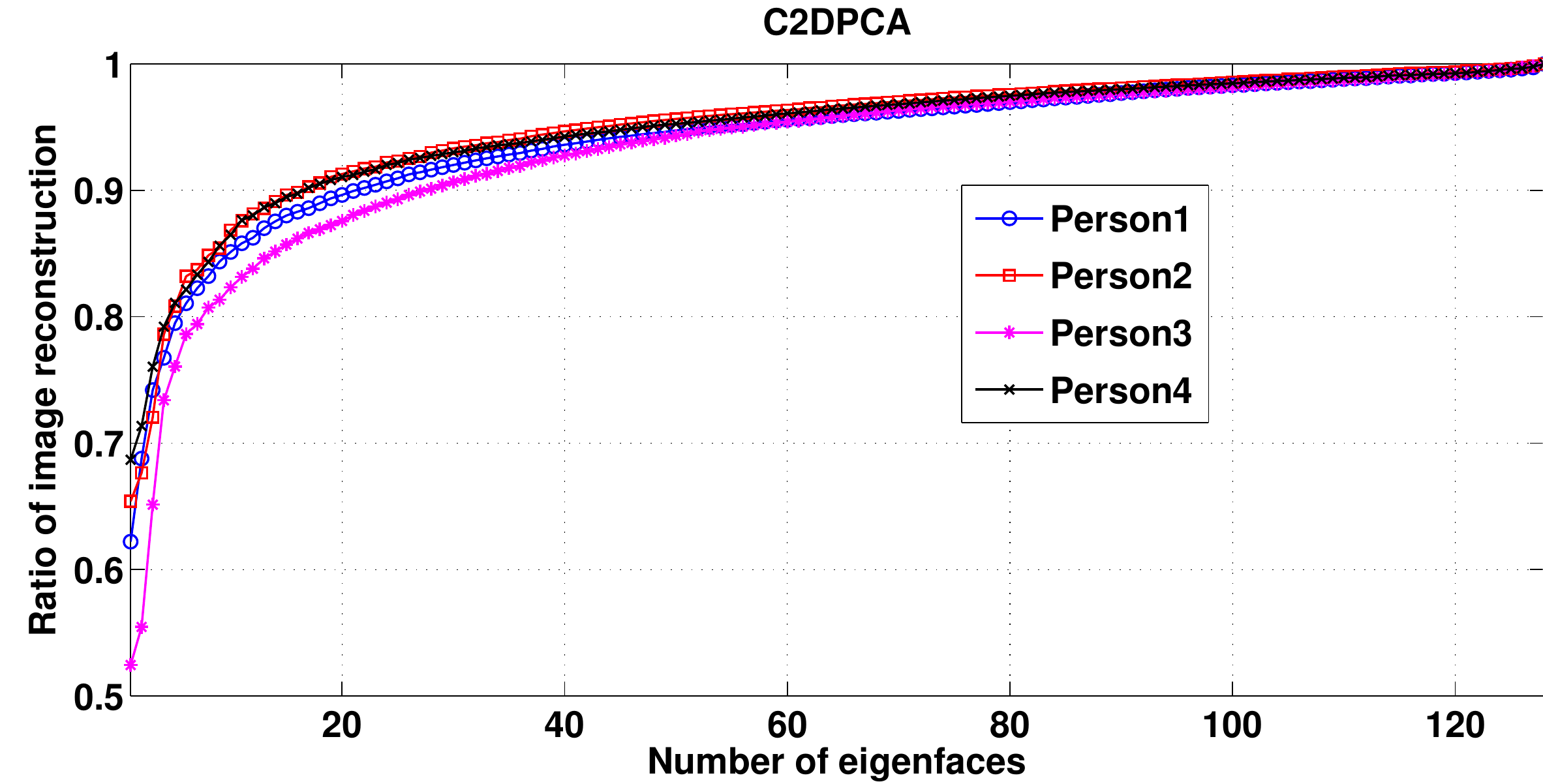}
    \includegraphics[height=0.5\textwidth,width=0.8\textwidth] {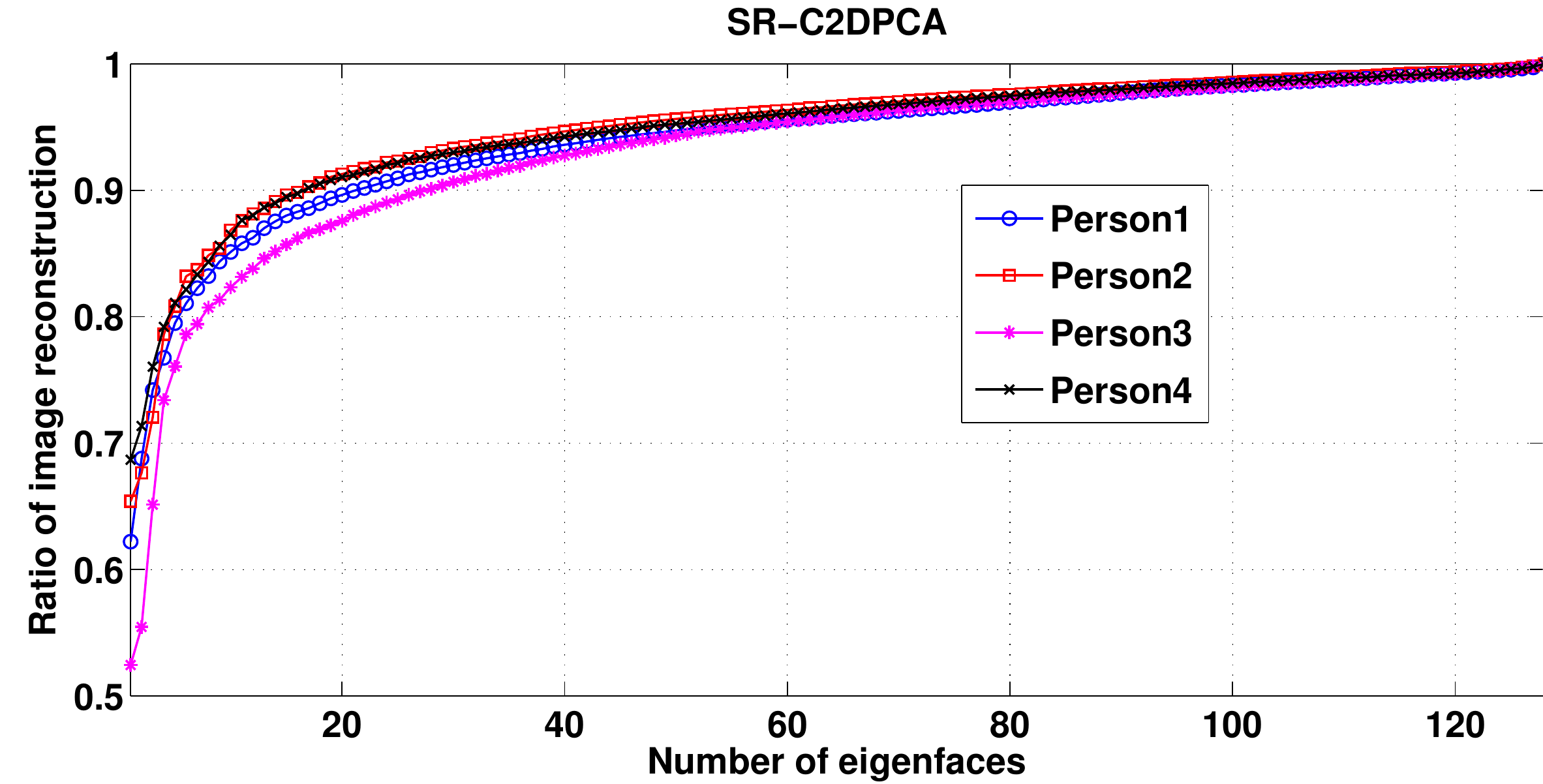}
\caption{Ratio of the reconstructed  images over the original face images.}
\label{f:Recon_ratios}
\end{figure*}

 \begin{figure*}[!t]
  \centering
    \includegraphics[height=0.5\textwidth,width=0.8\textwidth] {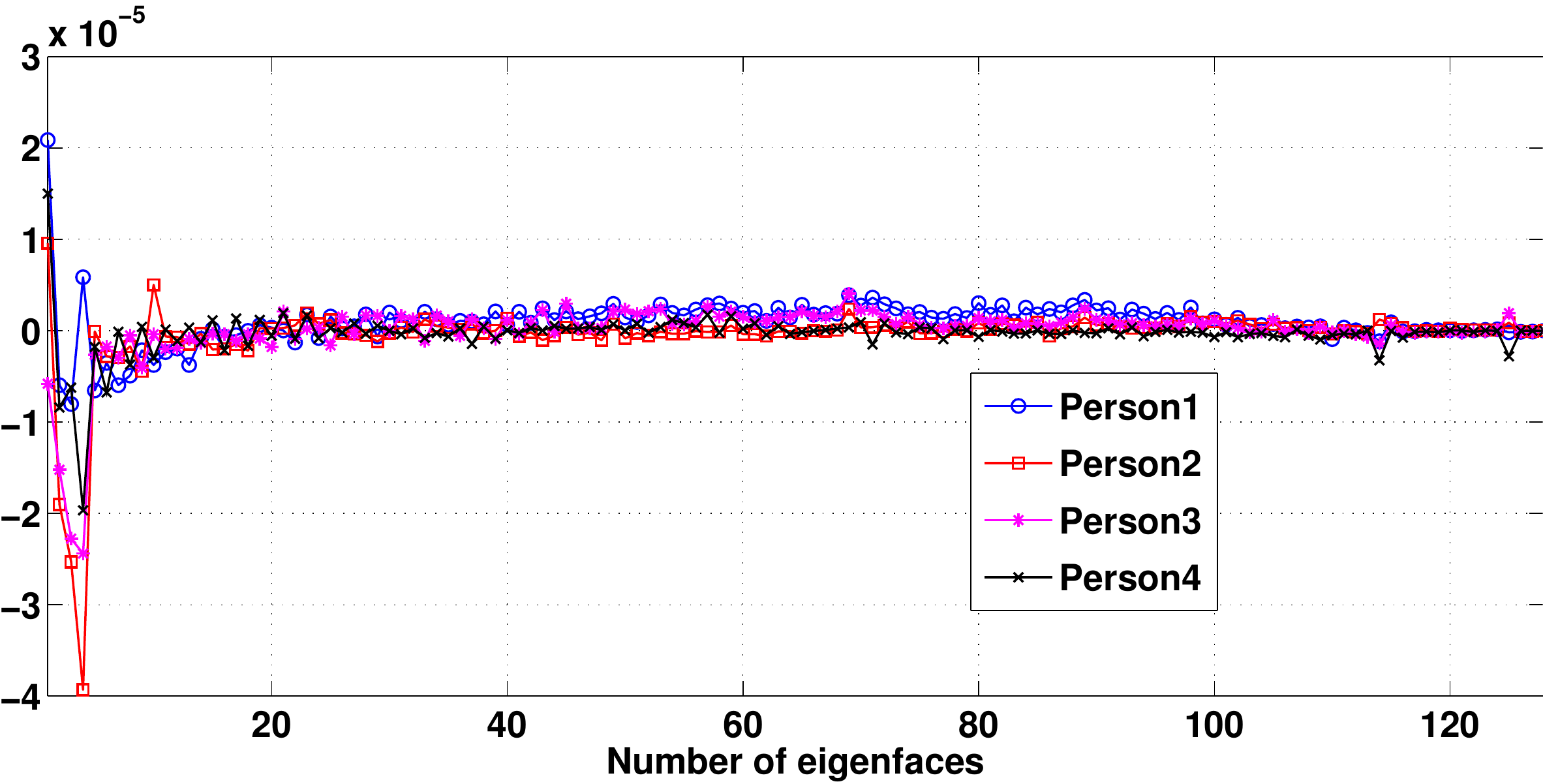}
\caption{The image reconstruction ratio of  SR-2DCPCA  subtracted by  that of 2DCPCA.}
\label{f:Recon_subratios}
\end{figure*}

\end{example}

\section{Conclusion}\label{s:conclusion}
\noindent
In this paper,  SR-2DCPCA is presented  for color face recognition and image reconstruction based on quaternion models.  This novel approach firstly applies label information of training samples, and emphasizes the role of training color face images with the same label which have a large variance.  The numerical experiments indicate that  SR-2DCPCA has a higher face recognition rate than  state-of-the-art methods and is effective in image reconstruction.



\end{document}